\documentclass[10pt]{article} % For LaTeX2e
\usepackage[accepted]{tmlr}
\usepackage{graphicx}
\usepackage{booktabs} % for professional tables
\usepackage{natbib}

\usepackage{etoolbox}
\usepackage{algorithm}
\let\classAND\AND
\let\AND\relax
\usepackage{algorithmic}
\let\AND\classAND
\AtBeginEnvironment{algorithmic}{\let\AND\algoAND}

\usepackage{amsmath}
\usepackage{amssymb}
\usepackage{mathtools}
\usepackage{amsthm}

\usepackage[capitalize,noabbrev]{cleveref}

\usepackage[utf8]{inputenc} % allow utf-8 input
\usepackage[T1]{fontenc}    % use 8-bit T1 fonts
\usepackage{url}            % simple URL typesetting
\usepackage{amsfonts}       % blackboard math symbols
\usepackage{nicefrac}       % compact symbols for 1/2, etc.
\usepackage{microtype}      % microtypography

\usepackage{algorithm}
\usepackage{algorithmic}
\usepackage{bbm}
\usepackage{bm}
\usepackage{color}
\usepackage{dsfont}
\usepackage{subcaption}
\usepackage{xspace}
\usepackage{apxproof}

\theoremstyle{plain}
\newtheorem{theorem}{Theorem}[section]
\newtheorem{proposition}[theorem]{Proposition}
\newtheorem{lemma}[theorem]{Lemma}

\theoremstyle{definition}

\theoremstyle{remark}

\usepackage[textsize=tiny]{todonotes}

\newcommand{\E}[1]{\mathbb{E} \left[#1\right]}
\newcommand{\condE}[2]{\mathbb{E} \left[#1 \,\middle|\, #2\right]}

\newcommand{\prob}[1]{\text{Pr} \left(#1\right)}

\mathchardef\mhyphen="2D
\newcommand{\rue}{\ensuremath{\tt RUE}\xspace}
\newcommand{\reucb}{\ensuremath{\tt ReUCB}\xspace}
\newcommand{\ucbe}{\ensuremath{\tt UCBE}\xspace}
\newcommand{\sh}{\ensuremath{\tt SH}\xspace}
\newcommand{\sr}{\ensuremath{\tt SR}\xspace}

\newcommand{\ttts}{\ensuremath{\tt TTTS}\xspace}

\title{UCB Exploration for Fixed-Budget Bayesian Best Arm Identification}
\author{\name Rong J.B. Zhu \email rongzhu@fudan.edu.cn\\
\addr Institute of Science and Technology for Brain-inspired Intelligence, Fudan University 
\AND
\name Yanqi Qiu \email yanqi.qiu@hotmail.com\\
\addr School of Mathematics and Statistics, Wuhan University
}

\def\month{10}  % Insert correct month for camera-ready version
\def\year{2024} % Insert correct year for camera-ready version
\def\openreview{\url{https://openreview.net/forum?id=BqSi73krYd}} % Insert correct link to OpenReview for camera-ready version

\begin{document}

\maketitle

\begin{abstract}
We study best-arm identification (BAI) in the fixed-budget setting. 
Adaptive allocations based on upper confidence bounds (UCBs), such as \ucbe, are known to work well in BAI. 
However, it is well-known that its optimal regret is theoretically dependent on instances, which we show to be an artifact in many fixed-budget BAI problems. 
In this paper we propose an UCB exploration algorithm that is both theoretically and empirically efficient for the fixed budget BAI problem under a Bayesian setting. 
The key idea is to learn prior information, which can enhance the performance of UCB-based BAI algorithm as it has done in the cumulative regret minimization problem. 
We establish bounds on the failure probability and the simple regret for the Bayesian BAI problem, 
providing upper bounds of order $\tilde{O}(\sqrt{K/n})$, up to logarithmic factors, where $n$ represents the budget and $K$ denotes the number of arms. 
Furthermore, we demonstrate through empirical results that our approach consistently outperforms state-of-the-art baselines.  
%code: https://github.com/rong-zhu/UCBExploration-BayesianBAI
\end{abstract}

\section{Introduction}
\label{sec:introduction}

We study best-arm identification (BAI) in stochastic multi-armed bandits \citep{Audibert:10,Karnin:13,Even:06,Bubeck:09,jamieson:14,Kaufmann:15}. In this problem, the learning agent sequentially interacts with the environment by pulling arms and receiving their rewards, which are sampled i.i.d.\ from their distributions. At the end, the agent must commit to a single arm. In the standard bandit setting, the agent maximizes its cumulative reward \citep{LR:85,Auer:02,lattimore19bandit,RongMattia:21}. In fixed-budget BAI \citep{Audibert:10,Karnin:13,JT:15,Li:18}, the agent maximizes the probability of choosing the best arm within a fixed budget. In fixed-confidence BAI, the agent minimizes the budget to attain a target confidence level for identifying the best arm \citep{Even:06,Audibert:10,Karnin:13}.
Here we focus on fixed-budget BAI. 

Adaptive allocations based on upper confidence bounds (UCBs) are known to work well in fixed-budget BAI. For example, \ucbe \citep{Audibert:10} is optimal, with failure probability decreasing exponentially up to logarithmic factors. 
However, it relies on a plug-in approach of an unknown problem complexity term, learning to the adaptive variant performing significantly worse \citep{Karnin:13}.
As a result, phase-based algorithms with uniform exploration in each phase, such as \emph{successive rejects (SR)} \citep{Audibert:10} and \emph{sequential halving (SH)} \citep{Karnin:13}, have been shown to work better in practice.
Furthermore, it should be noted that irrespective of the choice of the algorithms, i.e., UCB-based algorithms or phase-based algorithms, the optimal regret that decays exponentially are conditioned on the gaps between the maximal arm and the other arms not being small. If the gaps are small, the regret may decay polynomially instead of exponentially, as we will demonstrate in the next section. 
%Theoretically, no matter \sr, \sh, or \ucb, their optimal regret are also dependent on instances. 
%one limitation of \ucbe is that its optimal regret is dependent on instances, which means that its optimal rate is conditioned on gaps between the maximal arm and the other arms.  Additionally, 
%its optimal regret is dependent on instances, which means that its optimal rate is conditioned on gaps between the maximal arm and the other arms.
%These algorithms divide the budget into phases and eliminate suboptimal arms in each phase with uniform exploration, allowing for better performance in practice compared to UCB-based algorithms such as \ucbe.

It is well known that side information, such as the prior distribution of arm means, can improve the statistical efficiency of the cumulative regret minimization problem \citep{T:33,CL:11,RB:21}.  
Motivated by this, we propose a novel, theoretically and empirically efficient, and instance-independent UCB exploration algorithm for identifying the best arm by learning the prior information of arm means. 
We consider a Bayesian prior setting on arm means,  %\citep{RB:21}, 
where arm means are sampled i.i.d.\ from a Gaussian distribution, with mean $\mu_0$ and variance $\sigma_0^2$. 
The mean $\mu_0$ is shared among the arms. % and estimated from all of their observations. 
The variance $\sigma_0^2$ characterizes the spread of the arms. A lower $\sigma_0^2$ means that the optimal arm is harder to identify, since the gaps between the optimal and suboptimal arms are smaller on average. 
Our study shows that learning the prior of arm means also improves the performance of the UCB-based BAI algorithm and makes it more practical, as it has done in the cumulative regret minimization problem.

Further, we adopt \emph{random effect bandits} \citep{RB:21} to learn the prior information. 
From the random effect bandits, we obtain the posterior distribution of the arm means, then apply the UCB-based strategy for Bayesian BAI. 
The algorithm works as follows. In round $t\in[n]$, it pulls the arm with the highest UCB, observes its reward, and then updates the estimated arm means and their high-probability confidence intervals. We call it \emph{Random effect UCB Exploration (RUE)}. 
%The main difference from the classic \ucbe algorithm is that \rue constructs UCBs based on the uncertainties of the estimated arm means, implying that it is gap-free and statistically efficient. % from the \emph{random effect bandit}.

We make several contributions. 
First, we show that instance-dependence can compromise the optimality of the \ucbe algorithm, which can be considered as an artifact in many BAI problems. 
Second, we present an alternative formulation of the BAI problem that incorporates the prior distribution of arm means. 
%Our formulation also offers a measure for characterizing the hardness of the BAI problem in the Bayesian setting. 
Third, we bound the gap between the maximal arm and the others in probability.
This result provides a principled basis for Bayesian BAI. 
Fourth, we propose the efficient, practical, and instance-independent UCB exploration for the BAI problem, the \rue algorithm. 
Learning the prior information, \rue yields superior best-arm identification performance compared to state-of-the-art methods in empirical studies. 
Fifth, we analyze the failure probability and simple Bayes regret of \rue, and derive their upper bounds of $\tilde{O}(\sqrt{K/n})$, up to logarithmic factors. Here $n$ represents the budget and $K$ denotes the number of arms. 
Our analysis features a sharp bound on the prior gap through order statistics, and a careful comparison of the prior gap and confidence interval for bounding the error probability. 
Finally, we evaluate \rue empirically on a range of problems and observe that it outperforms sequential halving and successive rejects in broad domains, even works better than or similarly to the infeasible \ucbe in various domains. All proofs are in the appendix.

%%%%%%
%\subsection{On the Instance-Dependence}
\section{Exponentially Decaying Bounds in Fixed-budget BAI: An Artifact}
%\subsection{Small-Gap Issue of \ucbe}
\label{sec:onucbe}

Consider a fixed-budget BAI problem having $K$ arms with mean $\mu_k$, $k\in[K]$, and a horizon of $n$ rounds (or budgets). In round $t \in [n]$, the agent pulls arm $I_t\in [K]$  and observes its reward, drawn independently of the past. At the end of round $n$, the agent selects an arm $J_n$. The BAI problem concerns whether the final recommendation $J_n$ is the optimal one or not. For sake of simplicity, we will assume 
that there is a unique optimal arm. Let $i^* = \arg\max_{k \in [K]} \mu_k$ be the optimal arm and $\mu_* = \mu_{i^*}$.

Some BAI fixed-budget algorithms, such as \ucbe and \sh,  are considered (nearly) optimal since they can achieve an exponentially decaying failure probability that depends on the instance. However the property of exponentially decaying failure probability is conditioned. 
To illustrate this point, we will use \ucbe as an example. 
%\ucbe is considered optimal for fixed-budget BAI, as its failure probability decreases exponentially up to logarithmic factors \cite{Audibert:10}. 
%Define $\Delta_k=\mu_*-\mu_k$ and $\Delta_{\text{min}}=\min_{k\neq i^*}\mu_*-\mu_k$. 
%where $\Delta_{i^*}=\Delta_{\text{min}}$, 
\citet{Audibert:10} defines the problem complexity of the BAI problem 
$$H=\sum\limits_{k\in[K]}\Delta_k^{-2},$$ 
where $\Delta_k=\mu_*-\mu_k$ for $k\neq i^*$ and $\Delta_{i^*}=\min_{k\neq i^*}\mu_*-\mu_k$(denoted as $\Delta_{\text{min}}$),  
and shows that when the exploration degree is taken appropriately, the probability of error of \ucbe for a $K$-armed bandit with rewards in $[0,1]$ satisfies 
$$e_n\leq 2nK\exp \left[-\frac{n-K}{18H}\right].$$ 
However, the optimality of \ucbe depends on $H$, which relies on $\Delta_k$, 
particularly the minimum gap $\Delta_{\text{min}}$.  
Here, we emphasize the significant impact of the minimum gap $\Delta_{\text{min}}$. 
In situations where the minimum gap is small (i.e.,$\Delta_{\text{min}}\leq (54n^{-1}\log n)^{1/2}$, implying $H\geq 2\Delta_{\text{min}}^{-2}=n/(27\log n)$), the upper bound has 
$$2nK\exp \left[-\frac{n-K}{18H}\right]\approx 2nK\exp\left[-\frac{n}{18H}\right]\geq 2Kn^{-1/2}.$$ 
%meaning that the bound is actually larger than $O(Kn^{-1/2})$. 
Unfortunately, \emph{the small-gap condition} $\Delta_{\text{min}}\leq (54n^{-1}\log n)^{1/2}$ may not be small in practice. For instance, in a BAI problem with $n=10000$, the small-gap regime is defined by $\Delta_{\text{min}}\leq 0.223$ which is not considered small by any means.

Even when considering the lower bound, the \emph{small-gap} issue remains prevalent. 
In \citet{Audibert:10}, it is demonstrated that for Bernoulli rewards with parameters in $[p,1-p]$, where $p\in(0,1/2)$, the probability of error for \ucbe satisfies 
$$e_n\geq \exp \left[-\frac{(5+o(1))n}{p(1-p)H_2}\right],$$ 
where $H\leq H_2\leq \log(2K)H$ (see details in \citet{Audibert:10}). 
Consider an example where $p=0.2$. 
In cases where $\Delta_{\text{min}}\leq ((32n)^{-1}\log n)^{1/2}$, 
neglecting the contribution of the $o(1)$ term, the lower bound can be expressed as  
$$\exp \left[-\frac{5n}{p(1-p)H_2}\right]\geq \exp \left(-\frac{32n}{H}\right)\geq n^{-1/2}.$$ 
However, it is worth noting that the small-gap condition $\Delta_{\text{min}}\leq ((32n)^{-1}\log n)^{1/2}$ may not be small in practical scenarios. For instance, in a fixed-budget BAI problem with $n=1000$, the small-gap regime is defined by $\Delta_{\text{min}}\leq 0.0147$ which may not be considered very small in many fixed-budget BAI problems either.

Therefore, the exponentially decaying bounds on failure probability can be regarded as an artifact in many fixed-budget BAI problems. 

At last, the presence of the \emph{small-gap} problem also affects the choice of exploration degree, as the upper bound of \ucbe necessitates the parameter to be less than $25(n-K)/(36H)$. 
However, in these scenarios, this bound is on the order $\log n$, which leads to logarithmic exploration instead of linear exploration in the fixed-budget BAI problem.

\subsection{Scenario with Full-Information}
Now, let us consider a two-arm bandit problem and examine a scenario where we have complete information: each arm $k=1,2$ is pulled $n$ times and its outcomes are observed. 
We assume that $k=1$ is the optimal arm. 
In this instance, we can assume that the mean reward of each arm $k$ follows a normal distribution: 
$\bar{\mu}_k\sim \mathcal{N}(\mu_k,n^{-1}\sigma^2)$, 
where $\sigma^2$ represents the variance of reward noise. 
The probability of error can be bounded as follows: 
\begin{align}\label{full-e}
e_n^*:= &\text{Pr}(\bar{\mu}_{1}\leq \bar{\mu}_2) 
= \text{Pr}(\delta \geq  \Delta), 
\end{align}
where $\Delta=\mu_1-\mu_2$ and $\delta=(\bar{\mu}_{2}-\mu_2)-(\bar{\mu}_{1}-\mu_1)\sim  \mathcal{N}(0,2\sigma^2/n)$.
From Section 7.1 of \citet{Feller:68}, we have a lower bound of $e_n^*$: 
\begin{align}\label{full-e-lb}
e_n^*\geq & \left[\left(\frac{n\Delta^2}{2\sigma^2}\right)^{-1/2} - \left(\frac{n\Delta^2}{2\sigma^2}\right)^{-3/2}\right]\frac{1}{\sqrt{2\pi}}\exp\left(-\frac{n\Delta^2}{4\sigma^2}\right)
=:\bar{e}_n^*.
\end{align}

Now we show that this lower error bound \eqref{full-e-lb} also encounters the small-gap issue. 
Specifically, for the small-gap condition where $\Delta\leq \left(2\alpha n^{-1}\log n\right)^{1/2}$ with $\alpha$ controlling the gap size, 
the error bound in the full-information scenario is given by 
 $$\bar{e}_n^*=(2\pi)^{-1/2}[(\sigma^{-2}\alpha \log n)^{-1/2}-(\sigma^{-2}\alpha \log n)^{-3/2}]n^{-\alpha/(2\sigma^2)}.$$ 
 Assuming that $n$ satisfies $\sigma^2\alpha \log n>1$, we have that when $\alpha\geq \sigma^2$, 
 $$\bar{e}_n^*\geq (2\pi)^{-1/2}[1/\sqrt{\log n}-1/(\log n)^{3/2}]n^{-1/2}=\tilde{O}(n^{-1/2}).$$  
Therefore, even within the realm of infeasible full information in fixed-budget BAI problems,  
the exponentially decaying bounds on failure probability are also artifacts.  
Furthermore, the error order can exceed $\tilde{O}(n^{-1/2})$ 
depending on the gap size.

This scenario shows that achieving exponentially decaying regrets in fixed-budget BAI is infeasible when the small-gap issue is present. 
Therefore, the aim of this paper is not to develop an algorithm with exponentially decaying bounds, but rather to create an algorithm that significantly addresses the small-gap issue and achieves polynomially decaying regret bounds.

%%%
\section{A Bayesian Formulation for Best-Arm Identification}
\label{sec:rebandits}

%Specifically, we assume a reward-generating probability distribution $P_k(\mu_k, \nu^2)$ associated with each arm $k \in [K]$, where $P_k(\mu_k, \nu^2)$ is a reward-generating distribution with mean $\mu_k$ and variance $\nu^2$. 
In this paper we assume that the reward, denoted as $r_k$, associated with arm $k$, is generated from an (unknown) distribution with a mean $\mu_k$.  We assume that the reward noise, represented as $r_k-\mu_k$, adheres to a $\nu^2$-sub-Gaussian for a constant $\nu>0$. 
We introduce the assumption of random arm means on the BAI problem. 
Specifically, we assume that the mean arm reward $\mu_k$ of each arm $k \in [K]$ follows the following model 
\begin{equation}
  \label{payoff-M1}
  \mu_k
  = \mu_0+\delta_k\,,
\end{equation}
where $\delta_k\sim \mathcal{N}(0,\sigma_0^2)$ and $\mathcal{N}(0,\sigma_0^2)$ is a Gaussian distribution with zero mean and variance $\sigma_0^2$. As a result, the mean reward of arm $k$, $\mu_k$, is a stochastic variable with mean $\mu_0$ and variance $\sigma_0^2$. Recently, \citet{komiyama:2023} considers a Baysian BAI setting where they assume the uniform continuity of the conditional probability density functions. Different from theirs, we make a parametric perspective on priors $\mu_k$. 
Our model setting has $\sigma_0^2$ to represent the variability of the arm means. 
With a lower variance $\sigma_0^2$, the differences among the arms are smaller. 
Therefore, it is harder to learn the optimal arm, as the variability of the arm means is smaller. 
The priors $\mu_k$, $k\in[K]$, are taken into account through $(\mu_0,\sigma_0^2)$.

Different from the algorithms that rely on $H$, in the Bayesian BAI setting $\mu_*-\mu_k$ for $k\neq i^*$, can be arbitrarily small. Fortunately we can control the probability that the gap $\mu_*-\mu_k$ is less than $\alpha$ for any $\alpha>0$. This is our key in the Bayesian BAI problem. 
Define 
$$e_*(\alpha)=\prob{ \mu_{*}- \sup\limits_{k\neq i^*}\mu_k\leq \alpha},$$ for any $\alpha>0$. 
The probability $e_*$ represents the likelihood that the optimal arm $i^*$ is at least $\alpha$ better than the other arms. 
In other words, it reflects the probability of obtaining a gap between $i^*$ and the other arms that is less than $\alpha$ based on the prior distribution. In the BAI problem, $\alpha$ decreases as the numbers of pulls increases, allowing for control over the probability.
This highlights the inherent difficulty of the Bayesian BAI problem when dealing with the prior distribution of $(\mu_k)_{k\in[K]}$.

\begin{theorem}\label{lemma-true}
Assume $\mu_k$, for $k\in[K]$, are independently and identically distributed from $\mathcal{N}(\mu_0,\sigma_0^2)$. 
%There exists a constant $C>0$ such that for $\alpha>0$, 
Then for $\alpha>0$, 
\begin{equation*}
%e_*(\alpha)\leq C(\ln K)^{3/2}\alpha/\sigma_0, 
e_*(\alpha)\leq c_K\alpha/\sigma_0. 
\end{equation*}
where $c_K=4\sqrt{2}\ln K\sqrt{\ln \left(\frac{K}{4\sqrt{2\pi}\ln K}\right)} + \frac{2}{\sqrt{2\pi}}$.
\end{theorem}
%%%
This theorem provides a principled basis for Bayesian BAI in the following aspect. 
For any $\alpha>0$, the probability of bounding the gap by $\alpha$ is inversely proportional to the s.d. of arm means, $\sigma_0$, and is logarithmic of the number of arms, $K$.  It means that 
the effect of increasing $K$ on $e_*$ is negligible up to logarithmic factor.

%\subsection{Small $\sigma_0^2$ v.s. Small gap $\Delta_{\text{min}}$}
We end this section by comparing $\sigma_0^2$ with $\Delta_{\text{min}}$. 
In our Bayesian BAI setting, $\sigma_0^2=\mathbb{E}[(\mu_k-\mu_0)^2]$, which represents the expected value of the squared deviation of $\mu_k$ from $\mu_0$. 
This measure does not rely on the minimum gap $\Delta_{\text{min}}$. 
In other words, $\sigma_0^2$ can be $O(1)$ even if $\Delta_{\text{min}}=o(1)$. 
Typically, $\sigma_0^2=O(1)$, which largely avoids the \emph{small-gap} problem. 
Furthermore, as shown in Section \ref{sec:analysis}, our analysis accommodates a smaller $\sigma_0^2=O(1/K)$.

%\subsection{Estimation in Random Effect Bandits}
\section{Algorithm}
\label{sec:BAI-algorithm}

%\iffalse %%%
In \cref{sec:estimation} we show the Bayesian estimation,  
and provide a heuristic motivation for why the use of confidence intervals is applicable to Bayesian BAI. At last we propose a variant of the UCB algorithm in \cref{sec:BAI}.
%\fi %%%

\subsection{Estimation}
\label{sec:estimation}

For arm $k$ and round $t$, we denote by $T_{k, t}$ the number of its pulls by round $t$, and by $r_{k,1}, \dots,  r_{k,T_{k,t}}$ the sequence of its associated rewards. 

We use Gaussian likelihood function to design our algorithm. 
More precisely, suppose that the likelihood
of reward $r_{k,T_{k,t}}$ at time $t$, given $\mu_k$, were given by the pdf of Gaussian distribution
$\mathcal{N}(\mu_k,\sigma^2)$, where we take $\sigma^2=\delta^{-1}\nu^2$ for $0<\delta\leq 1$.  
We emphasize that the Gaussian likelihood model for rewards is only used above to design the algorithm. The assumptions on the actual reward distribution are the $\nu^2$-sub-Gaussian assumption. This setup is analogous to the one described in \cite{AG-Lin:13}. 
%Here $\sigma^2=\nu^2\delta^{-1}$ with $\delta\in(0,1)$ which parametrizes our algorithm.  

Let the history $H_t=(I_\ell, r_{I_\ell, T_{I_\ell, \ell}})_{\ell = 1}^{t-1}$. 
In the context where the prior for $\mu_k$ is given by $\mathcal{N}(\mu_0, \sigma_0^2)$, deriving the posterior distribution utilized by our algorithm is straightforward:
\begin{align}\label{posterior-mu}
\mu_{k} | H_t & \sim \mathcal{N}(\hat{\mu}_{k, t}, \tau_{k, t}^2).
\end{align}
Here, the posterior mean $\hat{\mu}_{k,t}$ of $\mu_k$ is given by 
\begin{equation}\label{blup}
\hat{\mu}_{k,t}=(1-w_{k,t})\bar{r}_{0,t}+w_{k,t}\bar{r}_{k,t}\,,
\end{equation}
where $w_{k,t}=\sigma_0^2/(\sigma_0^2+T_{k,t}^{-1}\sigma^2)$ and 
\begin{equation*}
\bar{r}_{0,t}=\left[\sum\limits_{k=1}^K(1-w_{k,t})T_{k,t}\right]^{-1} \sum\limits_{k=1}^K(1-w_{k,t})\sum\limits_{j=1}^{T_{k,t}}r_{k,j}\,. 
\end{equation*}
The posterior variance $\tau_{k,t}^2$ is given by 
\begin{align}\label{MSE}
\tau_{k,t}^2
&=\frac{w_{k,t}\sigma^2}{T_{k,t}} + \frac{(1-w_{k,t})^2 \sigma^2}{\sum\limits_{k=1}^K T_{k,t} (1-w_{k,t})},.
\end{align}

\iffalse %%%%
Comparing the estimator $\hat{\mu}_{k,t}$  with the sample average estimator $\bar{r}_{k,t}=T_{k,t}^{-1}\sum\limits_{j=1}^{T_{k,t}}r_{k,j}$, \cite{RB:21} showed 
that $\tau_{k,t}^2<\sigma^2 / T_{k,t}$ as long as $\sigma^2 > 0$ and $T_{k,t} \geq 1$. 
It means that tighter confidence interval is obtained by learning the prior than without using it. Therefore, using the prior may help to improve the performance of identifying the optimal arm.
%In other words, by using the weighted estimator $\hat{\mu}_{k,t}$ from the random effect model, less optimistic about arms is needed for the exploitation-exploration dilemma than using the classical per-arm estimators. This paper shows the help of the prior $\mu_k$ can also improve the performance of identifying the optimal arm. 
\fi %%%%

%\subsection{Heuristic Motivation}
%\label{sec:motivation}

%In fixed-budget BAI, we are concerned with identifying arm $J_n$ in round $n$ that is optimal. 
\iffalse 
Let the history $H_n=(I_\ell, r_{I_\ell, T_{I_\ell, \ell}})_{\ell = 1}^{n-1}$. 
Under the assumptions of Gaussian rewards $r_{k, j}\sim \mathcal{N}(\mu_k, \sigma^2)$ and $\mu_k\sim\mathcal{N}(\mu_0, \sigma_0^2)$, \cite{RB:21} showed that 
\begin{align}\label{posterior-mu}
\mu_{k} | H_n & \sim \mathcal{N}(\hat{\mu}_{k, n}, \tau_{k, n}^2).
\end{align}
\fi 
%(for completeness we show it in Lemma \ref{sec:mu0Dist})

%Let $i^* = \arg\max_{i \in [K]} \mu_k$ be the optimal arm and $\mu_* = \mu_{i^*}$. 
The following proposition motivates the use of confidence intervals in Bayesian BAI.
\begin{proposition}\label{var-lemma}
%Assume that $\hat{\mu}_{k,n}-\mu_k\sim \text{subG}(\tau_{k,n}^2)$ for $k \in [K]$, where $\text{subG}(\tau_{k,n}^2)$ is a sub-Gaussian distribution with variance proxy $\tau_{j,n}^2$. 
%Assume that $r_{k, j}\sim \mathcal{N}(\mu_k, \sigma^2)$ and $\mu_k\sim\mathcal{N}(\mu_0, \sigma_0^2)$ for $k \in [K]$. 
Let $\Delta_k = \mu_{i^*} - \mu_k$. For any sub-optimal arm $k\neq i^*$, 
%and any $\tau_{k,n}^2, \tau_{i^*,n}^2$, 
we have %, given $\Delta_k$ and $H_n$, 
$$
\text{Pr}(\hat{\mu}_{k,n}\geq \hat{\mu}_{i^*,n} | \Delta_k, H_n)
\leq \exp\left[-\frac{\Delta_k^2}{8\tau_{k,n}^2}\right]
+\exp\left[-\frac{\Delta_k^2}{8\tau_{i^*,n}^2}\right]. 
$$
\end{proposition}
%\begin{proof}
%The proposition is proved in Appendix \label{sec:proProof}.
%\end{proof}

\cref{var-lemma} shows that the probability of failing to identify the best arm depends on the gap $\Delta_k$ and $\tau_{k,n}^2$ for $k\in[K]$.  %the variances of the mean reward estimates of the arms.
This result shows that the widths of the confidence intervals affect the failure probability. Motivated by this observation, we design an efficient UCB-based BAI algorithm by using the estimates from random effect bandits.

\subsection{Random-Effect UCB Exploration}
\label{sec:BAI}

We apply random effect bandits to BAI, and propose a novel UCB-based exploration algorithm, called \emph{Random effect UCB Exploration (\rue)}. In \rue, the upper confidence bound of arm $k$ in round $t$ is
\begin{align*}
  U_{k, t} = \hat{\mu}_{k, t-1} + \sqrt{2\tau_{k,t-1}^2\log n}\,.
  %\text{ with } c_{k,t-1} =\sqrt{a\tau_{k,t-1}^2\log n}\,,
\end{align*} 
%where $a > 0$ is a tunable parameter. 
Different from \reucb \citep{RB:21}, which minimizes the cumulative regret and uses the degree of exploration $2\log t$, \rue uses the degree of exploration $2 \log n$, since BAI requires high-probability confidence intervals only at the final round. In round $t$, the algorithm pulls the arm with the highest UCB $I_t = \arg\max_{k \in [K]} U_{k, t}$ and collects the associated reward. Any fixed tie-breaking rule can be used for multiple maximal.

\begin{algorithm}[!ht]
\caption{\rue for best-arm identification.}
\label{alg:bai}
\begin{algorithmic}[1]
%\STATE Parameter: $a$
%\STATE Initialization: Randomly pull each arm twice 
\STATE Initialization: Pull each arm twice 
\FOR {$t = 2K+1, \dots,  n$}
\STATE Calculate  $U_{k,t}=\hat{\mu}_{k, t-1} + \sqrt{2\tau_{k, t-1}^2 \log n}$
\STATE Pull the arm with the highest $U_{k,t}$ for $k \in [K]$
\STATE Collect the reward associated the chosen arm
\ENDFOR
\STATE Return estimated best arm $J_n =\arg\max_{k\in[K]} {\hat{\mu}_{k,n}}$
\end{algorithmic}
\end{algorithm}

%\emph{Remark.} 
%Following \cite{RB:21}, the variances $\sigma_0^2$ and $\sigma^2$ can be estimated and updated in each round. As a result, we can replace $\sigma_0^2$ and $\sigma^2$ in \cref{alg:bai} by their corresponding estimated variances. 
%The algorithm is parameterized by the variances $\sigma_0^2$ and $\sigma^2$. 
In \rue, the priors $(\mu_k)_{k\in[K]}$ are taken into account through the variances $\sigma_0^2$ and $\sigma^2$. 
%Notably, there are close-form expressions to compute their estimators at each round \cite{RB:21}, and we provide these estimators of $\sigma_0^2$ and $\sigma^2$ in \cref{sec:estimation-variance} for completeness. 
Various methods for obtaining consistent estimators of $\hat{\sigma}_0^2$ and $\hat{\sigma}^2$ are available, including the method of moments, maximum likelihood, and restricted maximum likelihood. See \cite{Robinson:91} for details. 
One practical implication of this is that unlike \ucbe, our algorithm focuses on learning the prior to implement the algorithm. This feature of \rue can have surprising practical benefits.
%In contrast to \ucbe, the result has a surprising practical consequence: the agent does not need to estimate the prior, but learn the variances to implement the algorithm.

\section{Analysis}
\label{sec:analysis}

\iffalse %%%
In \cref{sec:analysis-freq}, we bound the probability that \rue fails to identify the best arm. 
The analysis is Bayesian, on average over the random arm means $(\mu_k)_{k\in[K]}$. 
This result follows the simple Bayes regret of \rue in \cref{sec:bayesregret}. 
\fi %%%
%We present two analyses. In \cref{sec:analysis-freq}, we bound the probability that \rue fails to identify the best arm. In \cref{sec:bayesregret}, we bound the simple Bayes regret of \rue. This latter analysis is under the assumption that \rue chooses the optimal arm proportionally to how many times it is pulled in $n$ rounds. Both analyses are Bayesian, on average over the random arm means $\mu_1, \dots, \mu_K$

%\subsection{Bayesian Failure Probability}
%\label{sec:analysis-freq}

We first bound the probability that \rue fails to identify the best arm. 
Let $e_n$ be the probability that \rue fails to identify the best arm
$$e_n=\text{Pr}(J_n\neq i^*),$$ 
which is over both the stochastic rewards and randomness in arm means $(\mu_k)_{k\in[K]}$. 
The main novelty in our analysis is 
%that we bound the minimum gap of priors $\mu_k$, $k\in[K]$ by integrating out the instance and 
that we control the failure probability of carefully comparing the gap of $(\mu_k)_{k\in[K]}$ and the confidence bounds. 

\begin{theorem}\label{them-R}
Consider Algorithm \ref{alg:bai} in a $K$-armed bandit with a budget $n\geq 4(K-1)$. 
Denote $\rho=\sqrt{(K(\sigma_0^2+\sigma^2)+\sigma_0^{-2}\sigma^2)/(K(\sigma_0^2+\sigma^2)+\sigma_0^2)}$ and $H_b=(K+\sigma_0^{-2}\sigma^2)\sigma^2$. 
Then the failure probability of Algorithm \ref{alg:bai} is 
\begin{align*}
e_n \leq &\gamma 
\sqrt{\frac{H_b(K-1)\log n}{nK}} + \gamma\sqrt{\frac{H_b\log n}{K(n-4(K-1))}} + 2Kn^{-mK+1}+ 2Kn^{-\frac{\sigma^2m}{\delta(2\sigma_0^2+\sigma^2)}+1}\,.
\end{align*} 
where $\gamma=2(1+4\rho)^{-1}(2+4\rho)c_K$ and $m=(1+K^{-1}\sigma_0^2/(\sigma_0^2+\sigma^2))(1+4\rho)^{-2}\sigma^{-2}\sigma_0^2$. 
\end{theorem}

Following the Bayesian failure probability of \rue in Theorem \ref{them-R}, we bound its simple Bayes regret 
\begin{align*}
  \mathrm{sr}_n
  = \mathbb{E}[\mu_* - \mu_{J_n}]\,,
\end{align*}
where the expectation is over stochastic rewards and the randomness in $(\mu_k)_{k\in[K]}$. 
By applying \cref{them-R}, we demonstrate in the following theorem that the simple Bayes regret is $\tilde{O}(\sqrt{K / n})$.%, up to logarithmic factors.  
\begin{theorem}\label{them-R-Bayes}
Consider Algorithm \ref{alg:bai} in a $K$-armed bandit with a budget $n> 4(K-1)$. 
Under the condition of Theorem \ref{them-R}, 
the simple Bayes regret of \rue is
\begin{align*}
  \mathrm{sr}_n
 \leq & \kappa\sqrt{\frac{2H_b(K-1)\log n}{nK}} 
 + \kappa\sqrt{\frac{2H_b\log n}{K(n-4(K-1))}}  + 4\sigma_0\sqrt{2\log K}\left(Kn^{-mK+1}+Kn^{-\frac{\sigma^2m}{\delta(2\sigma_0^2+\sigma^2)}+1}\right)\,.
\end{align*} 
where $\kappa=4(1+4\rho)^{-1}(2+4\rho)c_K\sqrt{\log K}$.
\end{theorem}

%%%%%
\subsection{Discussion}
\label{sec:analysis-discussion}

The parameter $\delta$ in \cref{them-R,them-R-Bayes} controls the last term of the bounds. When $\delta\leq\frac{\sigma^2m}{2(2\sigma_0^2+\sigma^2)}$, the last term is $\tilde{O}(K/n)$. 
As introduced in Section \ref{sec:estimation}, $\delta=\sigma^{-2}\nu^2$ 
denotes the ratio of $\nu^2$ to the variance $\sigma^2$ of the Gaussian likelihood designed in the algorithm.  
The condition on $\delta$ implies that a larger $\sigma^2$ than $\nu^2$ is required in the Gaussian likelihood. 
This requirement is analogous to that in \cite{AG-Lin:13}.
Moreover, $Kn^{-mK+1}=O(K/n)$ when $mK\geq 2$. 
Typically when $K\gg \sigma_0^2, \sigma^2$, 
we have $\rho\approx 1$, which impllies $m\approx \sigma_0^2\sigma^{-2}/25$. 
Clearly, the condition $mk\geq 2$ permits $\sigma_0^2=O(1/K)$.  
Hence, 
\cref{them-R,them-R-Bayes} state that the upper bounds on the failure probability and the simple Bayes regret are $\tilde{O}(\sqrt{H_b/n})$ respectively.

Here we characterize the hardness of the task using the quantity
$H_b=(K+\sigma_0^{-2}\sigma^2)\sigma^2$, which relies on the parameter set of bandits $(K,\sigma_0^2,\sigma^2)$. 
The quantity $H_b$ increases as the number of arms $K$ increases, the noise variance $\sigma^2$ increases, or the variability of the arms' means $\sigma^2$ decreases. 
Obviously, $H_b=O(K)$ when $\sigma_0^2=O(1)$ or $\sigma_0^2=O(1/K)$, given that $\sigma^2=O(1)$. 
In this case, the upper bounds are $\tilde{O}(\sqrt{K/n})$. 
This indicates that the bounds remain $\tilde{O}(\sqrt{K/n})$, even when $\sigma_0^2$ are small (i.e., $\sigma_0^2=O(1/K)$).

As demonstrated in Section \ref{sec:onucbe}, 
the upper bound on the failure probability of $\ucbe$ exhibits an exponential decay that is conditioned on the value of $\Delta_{\text{min}}$. 
When $\Delta_{\text{min}}$ is small, their bound degenerates to a polynomially decay, resulting in a failure probability of $\tilde{O}(Kn^{-\eta})$, where $\eta>0$ depends on how small $\Delta_{\text{min}}$ is. 
In contrast, the bounds provided in \cref{them-R,them-R-Bayes} are instance-independent, meaning that they solely depend on the budget $n$ and the bandit class defined by the number of arms and the variances, denoted as $(K,\sigma_0^2,\sigma^2)$, for which \rue is designed. These bounds are not influenced by the specific instances within the class, ensuring their independence from the particular characteristics of each instance. 
Nevertheless,  as shown in Section \ref{sec:rebandits}, 
$\sigma_0^2$ characterizes the variability of the arms' means $\mu_k$ in relation to $\mu_0$, 
thereby permitting very small $\Delta_{\text{min}}$. 
Consequently, the \emph{small-gap} issue does not pose a significant challenge in our algorithm.

\subsection{Proof Outline}
\label{sec:proof-outline}

Here we outline the proof. Comprehensive details can be found in the Appendix. 
Without loss of generality, we assume arm 1 is the optimal arm, i.e., $\mu_{i^*}=\mu_1$. 
Note that the initial round is $2K+1$, since every arm is pulled twice in the first $2K$ rounds.
Denote 
$$c_{k,t-1} =\sqrt{2\tau_{k,t-1}^2\log n}.$$
We define the events that all confidence intervals from round $2K+1$ to round $n$ hold as,
\begin{align*}
\mathcal{E} =\left\{\forall k \in [K], t\in\{2K+1,\dots,n\}: |\mu_{k} - \hat{\mu}_{k,t}| \leq \eta c_{k,t}\right\}\,,
\end{align*}
where $\eta=1/(1+4\rho)$.

The error probability is decomposed as
\begin{align}\label{regret-decomV-main}
&\prob{\hat{\mu}_{J_n,n}-\hat{\mu}_{1,n}>0}\notag\\
 = & \prob{(\hat{\mu}_{J_n,n}-\mu_{J_n})-(\hat{\mu}_{1,n}-\mu_1)>\Delta_{J_n} | \mathcal{E}}\prob{\mathcal{E}}
 +\prob{(\hat{\mu}_{J_n,n}-\mu_{J_n})-(\hat{\mu}_{1,n}-\mu_1)>\Delta_{J_n}|\bar{\mathcal{E}}}\prob{\bar{\mathcal{E}}}\notag\\
 \leq & \prob{\Delta_{J_n}< \eta(c_{1,n}+c_{J_n,n}) | \mathcal{E}}+\prob{\bar{\mathcal{E}}}\,,
\end{align} 
%%%%
From \eqref{regret-decomV-main}, $e_n$ is decomposed into two terms.  
The first term is to compare the prior's gap with the upper confidence bounds. The second term is the probability that the confidence intervals do not hold.

Denote $\beta=1+K^{-1}\sigma_0^{-2}\sigma^2$ and 
\begin{align}\label{tilde_c}
\tilde{c}_{k,n}=\sqrt{\frac{2\sigma_0^2\sigma^2\log n}{T_{k,n}\sigma_0^2+\sigma^2}}, \forall k\in[K].  
\end{align}
Because $\tilde{c}_{k,n}$ can be bounded by $c_{k,n}$: $c_{k,n}\leq \sqrt{\beta}\tilde{c}_{k,n}$ as shown in \cite{RB:21}. 
Thus, we have that
\begin{align*}%\label{toTilde}
\prob{\Delta_{J_n}< \eta(c_{1,n}+c_{J_n,n})}
\leq&\prob{\Delta_{J_n}< \eta\sqrt{\beta}(\tilde{c}_{1,n}+\tilde{c}_{J_n,n})}.
\end{align*}
Now we investigate to bound 
$$\prob{\Delta_{k}\leq \eta\sqrt{\beta}(\tilde{c}_{1,n}+\tilde{c}_{k,n})|\mathcal{E}}$$ for any $k\neq 1$. 
Denote $\Delta_{\text{min}}=\min\limits_{k\neq 1}\Delta_k$.
We define the following event of comparing $\tilde{c}_{1,n}$ with $\Delta_{\text{min}}$: 
\begin{equation*}
\mathcal{E}_1:=\{\tilde{c}_{1,n}\leq \Delta_{\text{min}}/(\sqrt{\beta}(1+\eta))\}. 
\end{equation*}
We can show that 
\begin{align*}%\label{prob-dk-v1}
&\prob{\Delta_k< \eta\sqrt{\beta}(\tilde{c}_{1,n}+\tilde{c}_{k,n})|\mathcal{E}}\notag\\
\leq & \prob{\bar{\mathcal{E}}_1|\mathcal{E}}= \prob{\Delta_{\text{min}}< (1+\eta) 
\sqrt{\frac{\beta\sigma_0^2\sigma^2\log n}{T_{1,n}\sigma_0^2+\sigma^2}}|\mathcal{E}}.
\end{align*}

Then we investigate $T_{1,n}$. We show that 
\begin{align}\label{eqn:T1n-main}
    T_{1,n}
    & \geq n-2(K-1)(1+\eta)^2\Delta_{\text{min}}^{-2}\beta\sigma^2\log n-2(K-1).
\end{align}
For decoupling $T_{1,n}$ and $\Delta_{\text{min}}$, 
we define the following event of controlling the minimum gap:
\begin{align*}%\label{E2}
\mathcal{E}_2:=\left\{\Delta_{\text{min}}\geq 2(1+\eta)\sqrt{(K-1)\beta\sigma^2n^{-1}\log n}\right\}. 
\end{align*}
On the event $\mathcal{E}_2$, \eqref{eqn:T1n-main} follows that $$\sigma_0^2T_{1,n}+\sigma^2\geq \sigma_0^2n/2-2\sigma_0^2(K-1)+\sigma^2.$$ 
Thus, we have that 
\begin{align*}%\label{prob-dk-v2}
&\prob{\Delta_{\text{min}}< (1+\eta) 
\sqrt{\frac{2\beta\sigma_0^2\sigma^2\log n}{T_{1,n}\sigma_0^2+\sigma^2}}|\mathcal{E}}\notag\\
\leq &\prob{\Delta_{\text{min}}< (1+\eta) 
\sqrt{\frac{2\beta\sigma_0^2\sigma^2\log n}{\sigma_0^2n/2-2\sigma_0^2(K-1)+\sigma^2}}|\mathcal{E}} +\prob{\bar{\mathcal{E}}_2|\mathcal{E}}\notag\\
\leq & 
\gamma \left(
\sqrt{\frac{\beta\sigma^2\log n}{\sigma_0^2n-4\sigma_0^2(K-1)+2\sigma^2}} + 
\sqrt{\frac{\beta\sigma^2\log n}{\sigma_0^2n/(K-1)}}\right),
\end{align*} 
where the last step is from applying Theorem \ref{lemma-true}. 
Therefore, the first term of the decomposition in \eqref{regret-decomV-main} is bounded.

At last, we investigate the second term of the decomposition in \eqref{regret-decomV-main}. 
Given our analysis of the algorithm, it is imperative to note that these models may be entirely unrelated to the actual reward distribution. Consequently, we cannot make the assumption that the posterior distribution of $\mu_k$—given the historical data—is Gaussian, with the posterior mean $\hat{\mu}_{k,t}$. Instead, we opt for a decomposition approach, as detailed below.
Denote $\beta_1=1+K^{-1}\sigma_0^2/(\sigma_0^2+\sigma^2)$. 
Define the following event:
\begin{align*}%\label{E2}
\mathcal{E}_{3k}:=\left\{|\bar{r}_{k,t}-\mu_k|\leq \eta\sqrt{\beta_1}\tilde{c}_{k,t}/2\right\}. 
\end{align*}
We have that 
\begin{align}\label{decomk-v2-main}
\prob{\bar{\mathcal{E}}}
\leq &\sum\limits_{k=1}^K\sum\limits_{t=1}^n\prob{\frac{\sigma^2|\bar{r}_{0,t}-\mu_k|}{T_{k,t}\sigma_0^2+\sigma^2}>\frac{\eta\tilde{c}_{k,t}}{2}} + \sum\limits_{k=1}^K\sum\limits_{t=1}^n\prob{\bar{\mathcal{E}}_{3k}}.
\end{align}
We aim to bound the two terms in \eqref{decomk-v2-main}. 
Utilizing the properties of sub-Gaussian distributions, 
it follows that $\bar{r}_{0,t}-\mu_0$ is sub-Gaussian.
By applying the sub-Gaussian tail inequality, we obtain that 
\begin{align*}%\label{e31-main}
\prob{\frac{\sigma^2|\bar{r}_{0,t}-\mu_k|}{T_{k,t}\sigma_0^2+\sigma^2}>\frac{\eta\tilde{c}_{k,t}}{2}}
\leq & \exp\left(-\frac{\beta_1\sigma_0^2K\log n}{(1+4\rho)^2\sigma^2}\right). 
\end{align*}
On the other hand, exploiting the sub-Gaussian property of $\bar{r}_{k,t}-\mu_k$, we have 
\begin{align*}%\label{decomk-v3-main}
\prob{\bar{\mathcal{E}}_{3k}}
&\leq \exp\left(-\frac{\beta_1\sigma_0^2\log n}{\delta(1+4\rho)^2(2\sigma_0^2+\sigma^2)}\right). 
\end{align*} 
Therefore, the theorem is proved.

%%%%%%%%%%%%%%%%%%%%
%%%%%%%%%%%%%%%%%%%%
\section{Experiments}
\label{sec:experiments}

We conduct two kinds of experiments.  In \cref{sec:experiments-random}, $\mu_k$ are random, where various settings are considered. 
In \cref{sec:experiments-fixed}, $\mu_k$ are fixed. Note that our modeling assumptions are violated here. We show these experiments of fixed $\mu_k$ because they are benchmarks established by \cite{Audibert:10} and \cite{Karnin:13}.

Our baselines include the state-of-the-art Successive Rejects (\sr) \citep{Audibert:10}, Sequential Halving (\sh) \citep{Karnin:13}, 
the UCB-exploration (\ucbe) \citep{Audibert:10}, and Top-Two Thompson sampling (\ttts) \citep{Russo:20}. As mentioned in Section \ref{sec:BAI}, in the implementation of \rue, we use the estimates of the variances $\sigma_0^2$ and $\sigma^2$. Therefore, no hyperparameters are required for \rue. 
\ucbe is implemented with parameter $a = 2n/H$, 
since this parameter works the best overall according to \cite{Audibert:10} and \cite{Karnin:13}. 
We do not report the adaptive variant of \ucbe because it performs much worse than \ucbe; even worse than \sh \citep{Karnin:13}.
Note \ucbe is infeasible since it requires the knowledge of a problem complexity parameter $H$, which depends on gaps. 
Additionally, we assess the two-stage algorithm proposed by \cite{komiyama:2023}, but it does not perform well in our experiments (see Figure \ref{fig:twostage} in the Appendix). Consequently, we exclude their method from our comparison.

\subsection{Random $\mu_k$}
\label{sec:experiments-random}

For random $\mu_k$, we have the following three setups: \\
(R1) Gaussian rewards with mean $\mu_k$ and variance $\sigma^2=1$, where $\mu_k\sim \mathcal{N}(0.5,0.1)$ for $k\in[K]$.\\  
(R2) The same $\mu_k$ as R1, but the rewards are Bernoulli rewards with means $\mu_k$. \\
(R3) Bernoulli rewards with $\mu_k\sim \mathcal{U}(0,0.5)$ for $k\in[K]$. \\
These setups allow us to explore how the performance of \rue compare against benchmarks under various distributions of noise and reward means.  
Although we assume a Gaussian distribution in our Bayesian BAI formulation, 
we also evaluate the performance of \rue for R3 when the assumption does not hold. 

We report the performance for $K=20$. Due to the randomness of $\mu_k$, the gaps and the difficulty of BAI varies. Therefore, we conduct our experiments on $50$ sampled $\mu_1, \dots, \mu_K$. For each set, we evaluate the performance of \rue and state-of-the-art algorithms, and then report the average performance. Like the case of fixed $\mu_k$, we also set $a=2$ through three random setups. 

The maximum budgets are set to $N=5000$ for R1 and R2, and to $N=12000$ for R3. We choose these values based on the median complexity terms in our experiments, which are $H \approx 2000$ for R1 and R2, and $H \approx 5500$ for R3. Then we study various budget settings $n \in \{N / 4, N / 2, N\}$, so that we can show how the performance varies when the budget is less than $H$ and double of $H$.

\begin{figure*}[!ht]
\centering
\begin{subfigure}[b]{0.32\columnwidth}
\includegraphics[keepaspectratio,width=1\linewidth]{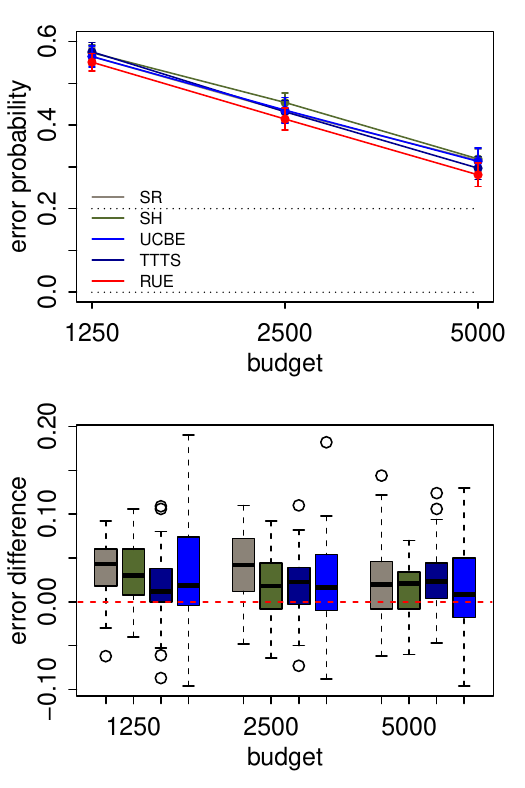}
\vspace{-0.5cm}
\subcaption{Setup R1}
\end{subfigure}
\begin{subfigure}[b]{0.32\columnwidth}
\includegraphics[keepaspectratio,width=1\linewidth]{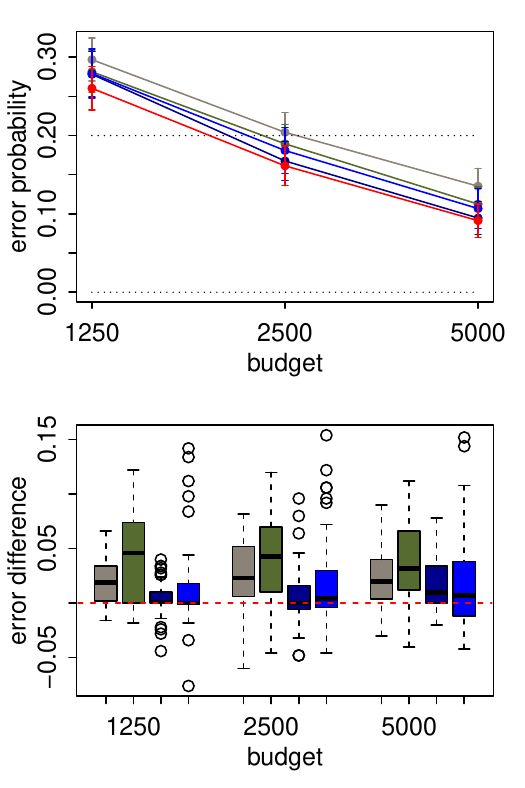}
\vspace{-0.5cm}
\subcaption{Setup R2}
\end{subfigure}
\begin{subfigure}[b]{0.32\columnwidth}
\includegraphics[keepaspectratio,width=1\linewidth]{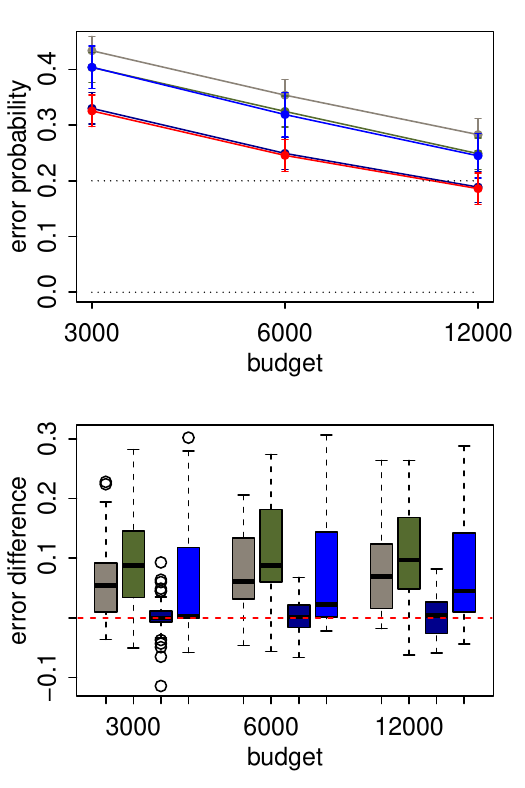}
\vspace{-0.5cm}
\subcaption{Setup R3}
\end{subfigure}
\vspace{-0.2cm}
\caption{Random $\mu_k$. 
Upper panel: the average performance among 50 experiments.  The bars denote the standard error of the mean among 50 experiments.
The lower panel: the error difference of the performance of the baselines, \sr, \sh, and \ucbe,  respectively, with respect to that of \rue among all 50 experiments. }
\label{fig:RandomSet}
\end{figure*}

In \cref{fig:RandomSet}, we report the average performance over $50$ experiments, and also boxplots of the relative performance of the baselines \sr, \sh, \ucbe, and \ttts with respect to \rue. We observe that, except for performing better or worse than \ttts in R3, \rue dominates other methods in all three setups, and even works better than \ucbe.  For example, for the case of $n=N/2$ in R3, the error probabilities of \rue are smaller 23\%, 24\%, and 31\%, respectively, than \ucbe, \sh, and \sr. 
These results shows the flexibility of \rue across various distributions of reward noise and across various distributions of reward means.

\subsection{Fixed $\mu_k$}
\label{sec:experiments-fixed}

Like \cite{Karnin:13}, we study six different experimental setups to comprehensively assess the \rue's performance:\\
(F1) One group of suboptimal arms: $\mu_k=0.45$ for $k\geq 2$.\\
(F2) Two groups of suboptimal arms: $\mu_k=0.45$ for $k=2,\dots, 8$ and $\mu_k=0.3$ otherwise.\\
(F3) Three groups of suboptimal arms: $\mu_k=0.48$ for $k=2,\dots, 5$, $\mu_k=0.4$ for $k=6,\dots,13$ and $\mu_k=0.3$ otherwise.\\
(F4) Arithmetic: The suboptimality of the arms form an arithmetic series where $\mu_2=0.5-1/(5K)$ and  $\mu_K=0.25$.\\
(F5) Geometic: The suboptimality of the arms form an geometric series where $\mu_2=0.5-1/(5K)$ and  $\mu_K=0.25$.\\
(F6) One real competitor:  $\mu_2=0.5-1/(10K)$ and $\mu_k=0.45$ for $k=3,\dots,K$.\\
In all setups, the reward distributions are Bernoulli and the mean reward of the best arm is $0.5$. The number of arms is $K = 20$. We also examine $K \in \{40, 80\}$ in \cref{fig:MoreArms} of Appendix, to show how \rue scales with $K$.

We set $2 \lceil H\rceil$ as the maximal budget for matching the hardness and for the limit of resources. Then we study various budget settings $n \in \{\lceil H/2\rceil, \lceil H\rceil, 2\lceil H\rceil\}$, so that we can show the performance when the budget is less or more than $H$. In \rue, we plug in the estimators of the variances $\sigma^2$ and $\sigma_0^2$ as in \cite{RB:21}.

\cref{fig:SixSets} shows results for our six problems. We have the following observations. First, \rue consistently outperforms \sh and \sr in all problems (except for the $n=\lceil H/2\rceil, \lceil H\rceil$ of F1, where it's a little worse than \sr). Take F2 as an example. For the budget $\lceil H/2\rceil$, $\lceil H\rceil$, and $2\lceil H\rceil\}$, the error probabilities of \rue are smaller 10\%, 20\%, and 66\%, respectively, than \sh. 
Second, comparing to the infeasible \ucbe, \rue outperforms it in F3,F5-F6, performs similarly to it in F4, and performs worse than it in F1 and F2. 
Third, comparing to \ttts, \rue outperforms it in most cases but performs slightly worse in $n=2\lceil H\rceil\}$ for F4 and F6.
Fourth, comparing with various $K \in \{40, 80\}$ in \cref{fig:MoreArms} of Appendix, we observe that the outperformance of \rue over others grows as $K$ increases. 
In summary, the observations suggest that \rue is expected to work well in various domains.

\begin{figure*}[!ht]
\centering
\includegraphics[keepaspectratio,width=\linewidth]{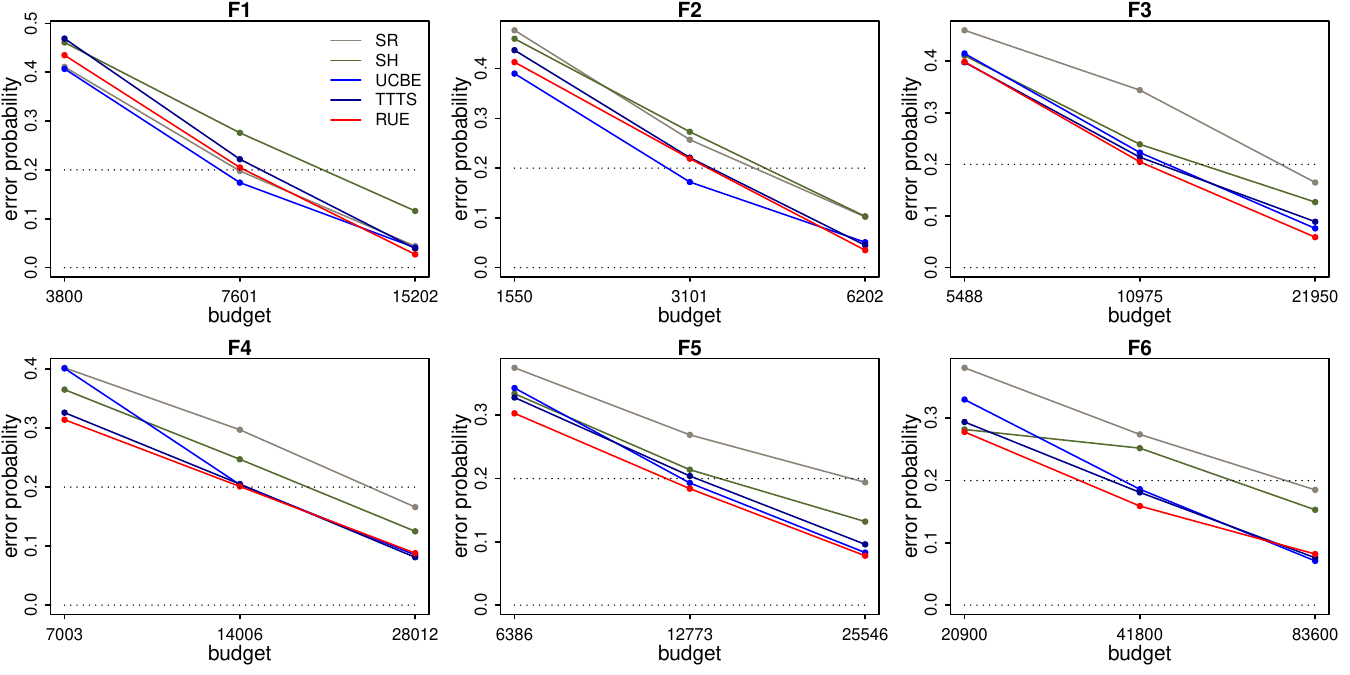}
%BAI_SixSets_v2.pdf
\vspace{-0.2cm}
\caption{Fixed $\mu_k$. 
All standard errors are less than $0.016$ and not reported. Each experiment is with budgets $H / 2$, $H$, and $2 H$ (labels of the x axis). The dotted lines denote error probabilities $0$ and $0.2$, just for visual clarity. 
}
\label{fig:SixSets}
\end{figure*}

\section{Related Work}
\label{sec:related}

\cite{Bubeck:09} showed that algorithms with at most logarithmic cumulative regret, such as UCB1 \citep{Auer:02}, are not suitable for BAI; and proposed to explore more aggressively using $O(\sqrt{n})$ confidence intervals. Motivated by it, \cite{Audibert:10} considered the fixed budget setting, where \ucbe and successive rejects are proposed for BAI. \ucbe and its adaptive version have $O(\sqrt{n})$ confidence intervals. In comparison, our work shows that a UCB-based algorithm with $O(\sqrt{\log n})$ confidence intervals performs well in BAI. On the other hand, the infeasible \ucbe algorithm depends on the unknown gap, while the estimated gap makes the adaptive \ucbe less efficient. Our algorithm does not rely on the actual or estimated gaps. 
\cite{Karnin:13} proposed sequential halving, which is popular in hyperparameter optimization \citep{JT:15,Li:18}.  Different from the method, this work focuses on efficient exploration based on upper confidence bounds.

Fixed-confidence setting was introduced by \cite{Even:06}, who proposed successive elimination for BAI. \cite{MT:04} derived tight distribution-dependent lower bounds for several variants of successive elimination, \cite{jamieson:14} proposed lil-UCB;  \cite{TNM:2017} extended lil-UCB to the KL-based confidence bounds \citep{Garivier:11,KK:13}, and \cite{Shang:20} adjusted TTTS \citep{Russo:20} for fixed-confidence guarantees. In comparison, we focus on the fixed-budget setting.

Although the fixed-confidence setting and the fixed-budget setting seem “dual” to each other,  they perform differently in several domains. 
Recently, \citet{Qin:22} proposed an open problem regarding whether there exists an algorithm other than uniform sampling itself that performs uniformly no worse than uniform sampling in the fixed-budget setting. \citet{Degenne:2023} and \citet{Wang:2023} demonstrated that in the fixed-budget setting, there are no such universally superior adaptive algorithms in several BAI problems. These studies show that the expected probability of error of BAI in the fixed-budget setting may be significantly different from that in the fixed-confidence setting. 
Our paper presents another observation in the fixed-budget setting: exponentially decaying bounds in fixed-budget BAI are an artifact.

\cite{RB:21,RB:22} proposed random effect bandits for cumulative regret minimization. Our work can be viewed as an extension of \cite{RB:21} to best-arm identification. 
We show that the prior information is helpful to develop an efficient, practical UCB exploration algorithm for the Bayesian BAI problem.

Recently, Bayesian BAI has received attention. 
\cite{Russo:20} proposed Bayesian algorithms, top-two variants of Thompson sampling (TTTS), that are tailored to identifying the best arm.  His analysis focused on the frequentist consistency and rate of convergence of the posterior distribution. \cite{Shang:20} followed his work and proposed a variant of TTTS for justifying its use for fixed-confidence guarantees. Different from theirs, we propose a variant of UCB exploration by using the prior information and show its efficiency by analyzing the error probability and the simple Bayes regret.
\cite{komiyama:2023} and \cite{Azizi:2023} derived a lower bound for this setting and a two-phase algorithm that matches it. However, empirically their algorithm works badly as shown in Appendix. 
Recently, \cite{Atsidakou:2022} introduced a Bayesian version of the SH algorithm and provided prior-dependent bound on the probability of error in multi-armed bandits. Furthermore, \cite{Nguyen:2024} established upper prior-dependent bounds on the expected probability of error of prior-informed BAI in Structured Bandits.

Our Bayesian BAI formulation is  also related to Bayesian optimization \citep{Snoek:12} which assumes a Gaussian process prior and updates the posterior with new observations. 
Similar to Bayesian optimization, the Bayesian BAI setting uses a Gaussian prior to learn configuration evaluations.  However, unlike Bayesian optimization, the Gaussian prior in the Bayesian BAI setting models reward means of individual arms.

\section{Conclusions}
\label{sec:conclusions}

We introduce a formulation of the Bayesian fixed-budget BAI problem by modeling the arm means, and propose \rue, an efficient, instance-independent UCB exploration for fixed-budget BAI. 
We empirically show that \rue outperforms \sh and \sr in broad domains, even works better than or similarly to the infeasible \ucbe in various domains. We derive $\tilde{O}(\sqrt{K/n})$ bounds on its Bayesian failure probability and simple Bayes regret. 
Inspired by \cite{Li:18}, which demonstrates that BAI can be applied to hyper-parameter optimization, our proposed \rue has potential for use in hyper-parameter optimization.
This application warrants serious investigation in the future.

Nevertheless, an inherent limitation of this study is the absence of a corresponding lower bound, 
as obtaining one for fixed-budget BAI is a challenging task \citep{Qin:22}. 
Another limitation is that the analysis in this paper only focuses on the Gaussian settings. Nevertheless, \rue does not assume any distributional form, since the setting in \eqref{payoff-M1} does not assume any particular distribution, but only assumes that the first- and second-order moments of $\mu_k$ are bounded. 
Moreover, our empirical results show that \rue works well in broad domains where the Gaussian assumptions are violated.
Therefore, an interesting question is to provide the bound on the failure probability under sub-Gaussian settings. 

In our analysis, we assume that $\sigma^2$ and $\sigma_0^2$ are known. However, in practice, we substitute these parameters with their estimates. Investigating the effect of this substitution in the Bayesian setting is extremely challenging since we need to integrate the error probability over the posterior of these parameters. We leave this challenging problem as a future direction for research.  Nevertheless, we expect this effect to be small since the agent's estimates of these parameters should converge to their true values as the agent gathers more data. 

\subsubsection*{Author Contributions}
Rong J.B. Zhu designed and performed the research, performed the analysis and experiments, and wrote the paper.  
Yanqi Qiu provided the proof for Lemma A.1. 

\subsubsection*{Acknowledgements}
We are very grateful to Branislav Kveton  for insightful discussions and invaluable comments on the preliminary version of this paper. 
We thank the Assigned Action Editor, Marcello Restelli, and the three reviewers for their insightful comments and suggestions that significantly improve this paper.

%\newpage

%\bibliographystyle{unsrt}
\bibliography{Bandit}
\bibliographystyle{tmlr}

\clearpage
%\onecolumn
\appendix
\section{Appendix}
\label{sec:appendix}

%\renewcommand{\thesubsection}{\Alph{subsection}}
%\setcounter{equation}{0}
%\renewcommand{\theequation}{S.\arabic{equation}}
%\setcounter{figure}{0}
%\renewcommand{\thefigure}{S.\arabic{figure}}

%%%%%%%%%%%%%%%%%%%%%%

\subsection{Proof of Theorem \ref{lemma-true}}
\label{sec:lemProof}

First we provide the following lemma, which provides a basic tool for bounding the gap $\mu_*-\mu_k$. Then we show that the probability $e_*$ depends on $\sigma_0^2$ and $\ln K$ by following Lemma \ref{lem:bound-D2}. 
\begin{lemma}\label{lem:bound-D2}
Assume $X_k$ for $k  = 1, \dots, K$, are independent and identically distributed from $\mathcal{N}(0,1)$. Denote $X_{(1)}\ge X_{(2)}\ge \cdots X_{(K)}$ the non-increasing re-ordering of $X_1, \dots, X_K$. Then there exists a constant $C> 0$, such that for all integers $K\ge 2$ and for $\alpha> 0$,
\[
\text{Pr}(X_{(1)}- X_{(2)} \le \alpha) \le C (\ln K)^{3/2} \alpha. 
\]
\end{lemma}

\begin{proof}
Denote $\eta(\alpha)= \text{Pr}(X_{(1)}- X_{(2)} \le \alpha)$. Let $\Phi(x)$ and $f(x)$ be the cumulative distribution function and the density function, respectively, of $X_i$. From the joint distribution of $(X_{(1)}, X_{(2)})$ (Fact 2 in the Appendix), we have that 
\begin{align*}
\eta(\alpha) & = K ( K-1)\int_{0\le x_1- x_2\le \alpha} \Phi(x_2)^{K-2} f(x_2)f(x_1)dx_2 dx_1
\\
&   = K ( K-1)\int_{0\le z \le \alpha; x_2 \in \mathbb{R}} \Phi(x_2)^{K-2} f(x_2)f(x_2+z)dx_2 dz
\\
&   = K ( K-1)\int_{x_2 \in \mathbb{R}} \Phi(x_2)^{K-2} f(x_2)[\Phi(x_2+\alpha)- \Phi(x_2)]dx_2 
\\
&= K \int_{x_2\in \mathbb{R}} \Phi(x_2+\alpha)d\Phi(x_2)^{K-1}-(K-1)\int_{x_2\in \mathbb{R}} d\Phi(x_2)^K
\\
& = K - K \int_{x_2\in \mathbb{R}} \Phi(x_2)^{K-1} f(x_2+ \alpha) dx_2 -(K-1)
\\
 & = 1 - K \int_{x_2\in \mathbb{R}} \Phi(x_2)^{K-1} f(x_2+\alpha)dx_2.
\end{align*}
It follows that 
\begin{align*}
\eta'(\alpha)& = -K  \int_{x_2\in \mathbb{R}} \Phi(x_2)^{K-1} f'(x_2+\alpha)dx_2.
\end{align*}
Fix $c>0$ (which will be choosen later on).  For any integer $K\ge 2$, let $a_K>0$ be the unique solution to the equation 
\[
a_K e^{a_K^2/2} = \frac{c K}{\ln K}.
\]
Set 
\begin{align*}
T_1(\alpha) &= - K \int_{x_2 > a_K}  \Phi(x_2)^{K-1} f'(x_2+\alpha)dx_2; 
\\
T_2(\alpha)&=- K \int_{x_2 \le a_K}  \Phi(x_2)^{K-1} f'(x_2+\alpha)dx_2.
\end{align*}
We can now write 
\[
\eta'(\alpha) = T_1(\alpha)+ T_2(\alpha).
\]
For the term $T_1(\alpha)$, by noting that $f'(x_2+\alpha)<0$ for all $x_2 > 0$ and $\Phi(x_2)\le 1$, we have 
\begin{align*}
T_1(\alpha) & \le -K \int_{x_2> a_K} f'(x_2+\alpha)dx
\\
&  = K f(a_K + \alpha) \le K f(a_K) = \frac{K}{\sqrt{2\pi}} e^{-a_K^2/2}. 
\end{align*}
For the term $T_2(\alpha)$, we have 
\begin{align*}
| T_2(\alpha)| &\le K \sup_{x_2 \le a_K} \Phi(x_2)^{K-1} \int_{x_2\le a_K} |f'(x_2+\alpha)|dx_2
\\
& \le K \Phi(a_K)^{K-1}\int_{x_2\in \mathbb{R}} |f'(x_2)| dx_2 =  \frac{2K}{\sqrt{2\pi}} \Phi(a_K)^{K-1}.
\end{align*}
Using the classical estimate 
\[
\Phi(t)\le 1-  \frac{1}{\sqrt{2\pi}} \frac{t}{t^2+1} e^{-t^2/2},
\]
we obtain 
\begin{align*}
\eta'(\alpha)& \le \frac{K}{\sqrt{2\pi}} e^{-a_K^2/2} + \frac{2K}{\sqrt{2\pi}} \Phi(a_K)^{K-1}
\\
& \le  \frac{K}{\sqrt{2\pi}} e^{-a_K^2/2} + \frac{2K}{\sqrt{2\pi}} \left(1-  \frac{1}{\sqrt{2\pi}} \frac{a_K}{a_K^2+1} e^{-a_K^2/2}\right)^{K-1}.
\end{align*}
By our choice of $a_K$, we have  $e^{-a_K^2/2} = \frac{a_K \ln K}{cK}$, hence
\begin{align*}
\eta'(\alpha)& \le   \frac{a_K \ln K }{c \sqrt{2\pi}}  + \frac{2K}{\sqrt{2\pi}} \left(1-  \frac{1}{\sqrt{2\pi}} \frac{a_K^2}{a_K^2+1}  \frac{\ln K}{c K}\right)^{K-1}.
\end{align*}
For $K$ large enough such that $cK /\ln K \ge \sqrt{e}$, we have $ a_K^2 e^{a_K^2} \ge e $
and thus  $a_K> 1$. Then
\[
\frac{c K}{\ln K}  = a_K e^{a_K^2/2} \ge e^{a_K^2/2}.
\]
It follows that for $K$ large enough, we have 
\[
1\le a_K\le \sqrt{2 \ln\left(\frac{c K}{\ln K}\right)}.
\]
Consequently, for $K$ large enough and $\alpha > 0$,
\[
\eta'(\alpha) \le  \frac{\ln K}{c \sqrt{\pi}} \sqrt{ \ln\left(\frac{c K}{\ln K}\right)} +  \frac{2K}{\sqrt{2\pi}}  \left( 1- \frac{1}{2 \sqrt{2\pi}} \frac{\ln K}{c K}\right)^{K-1}.
\]
Using the elementary inequality: for any $x>1$, 
\[
(1 - \frac{1}{x})^x = \exp (x \ln (1- 1/x))= \exp(-x \sum\limits_{k=1}^\infty  x^{-k}/k) \le e^{-1},
\]
we obtain that, for any integer $K\ge 2$, 
\begin{align*}
\frac{2K}{\sqrt{2\pi}}  \left( 1- \frac{1}{2 \sqrt{2\pi}} \frac{\ln K}{c K}\right)^{K-1} & = \frac{2K}{\sqrt{2\pi}}  \left[\left( 1- \frac{1}{2 \sqrt{2\pi}} \frac{\ln K}{c K}\right)^{\frac{2\sqrt{2\pi} c K}{\ln K}}\right]^{\frac{(K-1)\ln K}{2\sqrt{2\pi} c K}}
\\
&\le \frac{2K}{\sqrt{2\pi}}  e^{-\frac{(K-1)\ln K}{2\sqrt{2\pi} c K}} = \frac{2}{\sqrt{2\pi}} K^{1- \frac{1}{2\sqrt{2\pi}c} + \frac{1}{2K \sqrt{2\pi} c}} 
\\
& \le \frac{2}{\sqrt{2\pi}} K^{1- \frac{1}{2\sqrt{2\pi}c } + \frac{1}{4 \sqrt{2\pi}c}} = \frac{2}{\sqrt{2\pi}} K^{1-  \frac{1}{4 \sqrt{2\pi}c}}.
\end{align*}
Now let us take 
\[
c = \frac{1}{4\sqrt{2\pi}}.
\]
We obtain that, for $K$ large enough, 
\[
\eta'(\alpha) \le \frac{\ln K}{c\sqrt{\pi}} \sqrt{\ln \left(\frac{cK}{\ln K}\right)} + \frac{2}{\sqrt{2\pi}}. 
\]
By the mean value theorem, we obtain 
\[
\eta(\alpha) \le  \left(\frac{\ln K}{c\sqrt{\pi}} \sqrt{\ln \left(\frac{cK}{\ln K}\right)} + \frac{2}{\sqrt{2\pi}}\right)\alpha. 
\]
This clearly implies the desired result. 
\end{proof}

Denote $\mu_{(1)}\ge \mu_{(2)}\ge \cdots \mu_{(K)}$ the non-increasing re-ordering of $\mu_1, \dots, \mu_K$. 
We have 
\begin{align*}
e_*(\alpha)&= \prob{\mu_{(1)}-\mu_{(2)}<\alpha}\notag\\
&=\prob{\sigma_0^{-1}(\mu_{(1)}-\mu_{(2)})<\sigma_0^{-1}\alpha}.
\end{align*}
Then the theorem is a direct result of Lemma \ref{lem:bound-D2}.

%%%%%%%%%%%%%%
\subsection{Proof of Proposition \ref{var-lemma}}
\label{sec:proProof}
Let $v$ be the midpoint between $\mu_k$ and $\mu_{i^*}$. Then %given $\Delta_k$ and $H_n$,
\begin{align*}
  \quad \text{Pr}(\hat{\mu}_{k, n} \geq \hat{\mu}_{i^*, n} | \Delta_k, H_n)
  & = \text{Pr}(\hat{\mu}_{i^*, n} > v)\text{Pr}(\hat{\mu}_{k, n} \geq \hat{\mu}_{i^*, n} | \hat{\mu}_{i^*, n} > v) \notag\\
  & \quad + \text{Pr}(\hat{\mu}_{i^*, n} \leq v)\text{Pr}(\hat{\mu}_{k, n} \geq \hat{\mu}_{i^*, n} | \hat{\mu}_{i^*, n} \leq v) \\
  & \leq \text{Pr}(\hat{\mu}_{k, n} \geq v) + \text{Pr}(\hat{\mu}_{i^*, n} \leq v) \\
  & = \text{Pr}(\hat{\mu}_{k, n} - \mu_k \geq \Delta_k / 2) +
  \text{Pr}(\hat{\mu}_{i^*, n} - \mu_{i^*} \leq - \Delta_k / 2)\\
  & \leq \exp\left[-\Delta_k^2/(8\tau_{k,n}^2)\right]
+\exp\left[-\Delta_k^2/(8\tau_{i^*,n}^2)\right],
\end{align*} 
%\todob{Make sure that all probabilities are written in the same way.}
where the last step is a direct result of the Gaussian tail bound shown in \cref{sec:property}.

%%%%
\subsection{Proof of \cref{them-R}}
\label{sec:proofthem1}

%\begin{proof}
%We sketch how to prove Theorem \ref{them-R} without going into full technical detail. See \cref{sec:them1} for the detail. 
Without loss of generality, we assume arm 1 is the optimal arm, i.e., $\mu_{i^*}=\mu_1$. 
Note that the initial round is $2K+1$, since every arm is pulled twice in the first $2K$ rounds.
Denote 
$$c_{k,t-1} =\sqrt{2\tau_{k,t-1}^2\log n}.$$
We define the events that all confidence intervals from round $2K+1$ to round $n$ hold as,
%for a constant $\eta\in (0,1)$, 
\begin{align*}
\mathcal{E} =\left\{\forall k \in [K], t\in\{2K+1,\dots,n\}: |\mu_{k} - \hat{\mu}_{k,t}| \leq \eta c_{k,t}\right\}\,,
\end{align*}
where $\eta=1/(1+4\rho)$.

The error probability is decomposed as
\begin{align}\label{regret-decomV}
 \prob{\hat{\mu}_{J_n,n}-\hat{\mu}_{1,n}>0}
& = \prob{(\hat{\mu}_{J_n,n}-\mu_{J_n})-(\hat{\mu}_{1,n}-\mu_1)>\Delta_{J_n} | \mathcal{E}}\prob{\mathcal{E}}\notag\\
&\quad +\prob{(\hat{\mu}_{J_n,n}-\mu_{J_n})-(\hat{\mu}_{1,n}-\mu_1)>\Delta_{J_n}|\bar{\mathcal{E}}}\prob{\bar{\mathcal{E}}}\notag\\
& \leq \prob{\Delta_{J_n}< \eta(c_{1,n}+c_{J_n,n}) | \mathcal{E}}+\prob{\bar{\mathcal{E}}}\,,
\end{align} 
%%%%
From \eqref{regret-decomV}, $e_n$ is decomposed into two terms.  
The first term is to compare the prior's gap with the upper confidence bounds. The second term is the probability that the confidence intervals do not hold.  
%Our main analysis is on the first term. While for the second term, we use the similar idea as \cite{RB:21} to bound it. 

At first, we focus on investigating the first term of the decomposition in \eqref{regret-decomV}. 
Denote $\beta=1+K^{-1}\sigma_0^{-2}\sigma^2$ and $\beta_1=1+K^{-1}\sigma_0^2/(\sigma_0^2+\sigma^2)$. 
Noting that $\tilde{c}_{k,n}$ just relies on its corresponding $T_{k,n}$, using $\tilde{c}_{k,n}$ instead of $c_{k,n}$ breaks the dependence of arm $k$ on other arms.  
Because $\tilde{c}_{k,n}$ can be bounded by $c_{k,n}$: $c_{k,n}\leq \sqrt{\beta}\tilde{c}_{k,n}$ as shown in \cite{RB:21}. 
Thus, we have that
\begin{align}\label{toTilde}
\prob{\Delta_{J_n}< \eta(c_{1,n}+c_{J_n,n})}
&\leq\prob{\Delta_{J_n}< \eta\sqrt{\beta}(\tilde{c}_{1,n}+\tilde{c}_{J_n,n})}.
\end{align}
Now we bound $\prob{\Delta_{k}\leq \eta\sqrt{\beta}(\tilde{c}_{1,n}+\tilde{c}_{k,n})|\mathcal{E}}$ for any $k\neq 1$. 
Denote $\Delta_{\text{min}}=\min\limits_{k\neq 1}\Delta_k$.
We define the following event of comparing $\tilde{c}_{1,n}$ with $\Delta_{\text{min}}$: 
\begin{equation*}
\mathcal{E}_1:=\{\tilde{c}_{1,n}\leq \Delta_{\text{min}}/(\sqrt{\beta}(1+\eta))\}. 
\end{equation*}
We have that 
\begin{align}\label{prob-dk}
&\quad \prob{\Delta_k< \eta\sqrt{\beta}(\tilde{c}_{1,n}+\tilde{c}_{k,n})|\mathcal{E}}\notag\\
&\leq \prob{\Delta_k< \eta\Delta_{\text{min}}/(1+\eta)+\eta\sqrt{\beta}\tilde{c}_{k,n}|\mathcal{E}_1,\mathcal{E}}+\prob{\bar{\mathcal{E}}_1|\mathcal{E}}\notag\\
& \leq \prob{\Delta_k< \eta(1+\eta)\sqrt{\beta}\tilde{c}_{k,n}|\mathcal{E}_1,\mathcal{E}}+\prob{\bar{\mathcal{E}}_1|\mathcal{E}}.
\end{align}
We shall show $\Delta_k\geq \eta(1+\eta)\sqrt{\beta}\tilde{c}_{k,n}$ given $\mathcal{E}$ and $\mathcal{E}_1$  when $\eta$ satisfies 
$2\eta(1+\eta)\sqrt{\beta/\beta_1}+\eta-1\leq0$.
In the following we take $\eta=1/(1+4\rho)$. 
Lemma \ref{TwoBoundOnC} shows that on the event $\mathcal{E}$ the following result holds:
\begin{align*}
\Delta_k\geq (1-\eta)\sqrt{\beta_1}\tilde{c}_{k,t}- (1+\eta)\sqrt{\beta}\tilde{c}_{1,t},
\end{align*} 
implying that, on the events $\mathcal{E}$ and $\mathcal{E}_1$, 
\begin{equation*}
\Delta_k\geq (1-\eta)\sqrt{\beta_1}\tilde{c}_{k,t}/2\geq \eta(1+\eta)\sqrt{\beta}\tilde{c}_{k,n},
\end{equation*}
where the second inequality is from $\eta$ satisfying 
$2\eta(1+\eta)\sqrt{\beta/\beta_1}+\eta-1\leq0$. 
Thus, \eqref{prob-dk} follows that 
\begin{align}\label{prob-dk-v1}
\prob{\Delta_k< \eta\sqrt{\beta}(\tilde{c}_{1,n}+\tilde{c}_{k,n})|\mathcal{E}}
& \leq \prob{\bar{\mathcal{E}}_1|\mathcal{E}}= \prob{\Delta_{\text{min}}< (1+\eta) 
\sqrt{\frac{\beta\sigma_0^2\sigma^2\log n}{T_{1,n}\sigma_0^2+\sigma^2}}|\mathcal{E}}.
\end{align}

Now we investigate $T_{1,n}$. 
Lemma \ref{TwoBoundOnC} shows that on the event $\mathcal{E}$ for $k\neq 1$,
\begin{align*}
T_{k,t}&\leq 2+2\Delta_k^{-2}(1+\eta)^2\beta\sigma^2\log n.%\label{Tk_up}.
\end{align*}
It follows that 
\begin{align}\label{eqn:T1n}
    T_{1,n}&=n-\sum\limits_{k\neq 1}T_{k,n} \geq n-2(K-1)(1+\eta)^2\Delta_{\text{min}}^{-2}\beta\sigma^2\log n-2(K-1).
\end{align}
For decoupling $T_{1,n}$ and $\Delta_{\text{min}}$, 
we define the following event of controlling the minimum gap:
\begin{align*}%\label{E2}
\mathcal{E}_2:=\left\{\Delta_{\text{min}}\geq 2(1+\eta)\sqrt{(K-1)\beta\sigma^2n^{-1}\log n}\right\}. 
\end{align*}
On the event $\mathcal{E}_2$, \eqref{eqn:T1n} follows that $$\sigma_0^2T_{1,n}+\sigma^2\geq \sigma_0^2n/2-2\sigma_0^2(K-1)+\sigma^2.$$ 
Thus, we have that 
\begin{align}\label{prob-dk-v2}
&\prob{\Delta_{\text{min}}< (1+\eta) 
\sqrt{\frac{2\beta\sigma_0^2\sigma^2\log n}{T_{1,n}\sigma_0^2+\sigma^2}}|\mathcal{E}}\notag\\
\leq &\prob{\Delta_{\text{min}}< (1+\eta) 
\sqrt{\frac{2\beta\sigma_0^2\sigma^2\log n}{\sigma_0^2n/2-2\sigma_0^2(K-1)+\sigma^2}}|\mathcal{E}} +\prob{\bar{\mathcal{E}}_2|\mathcal{E}}\notag\\
= &\prob{\Delta_{\text{min}}< (1+\eta) 
\sqrt{\frac{2\beta\sigma_0^2\sigma^2\log n}{\sigma_0^2n/2-2\sigma_0^2(K-1)+\sigma^2}}}\notag\\
& +\prob{\Delta_{\text{min}}< (1+\eta)\sqrt{2a(K-1)\beta\sigma^2n^{-1}\log n}}\notag\\
\leq & 
c_K(1+\eta)\left(
\sqrt{\frac{2\beta\sigma_0^2\sigma^2\log n}{\sigma_0^2n/2-2\sigma_0^2(K-1)+\sigma^2}} + 
\sqrt{\frac{4\beta\sigma^2\log n}{n/(K-1)}}\right),
\end{align} 
where the last step is from Lemma \ref{lem:bound-D2}.

At last, we investigate the second term of the decomposition in \eqref{regret-decomV}. 
Define the following event: for each arm $k$
\begin{align*}%\label{E2}
\mathcal{E}_{3k}:=\left\{|\bar{r}_{k,t}-\mu_k|\leq \eta\sqrt{\beta_1}\tilde{c}_{k,t}/2\right\}. 
\end{align*}
From \eqref{posterior-mu}, 
we have that 
\begin{align}\label{decomk-v2}
\prob{\bar{\mathcal{E}}}
&\leq\sum\limits_{k=1}^K\sum\limits_{t=1}^n\prob{|\hat{\mu}_{k,t}-\mu_k|> \eta c_{k,t}}\notag\\
&=\sum\limits_{k=1}^K\sum\limits_{t=1}^n\prob{|\sigma^2/(T_{k,t}\sigma_0^2+\sigma^2)(\bar{r}_{0,t}-\mu_k)+w_{k,t}(\bar{r}_{k,t}-\mu_k)|> \eta c_{k,t}}\notag\\
&\leq \sum\limits_{k=1}^K\sum\limits_{t=1}^n\left[\prob{\sigma^2/(T_{k,t}\sigma_0^2+\sigma^2)|\bar{r}_{0,t}-\mu_k|>\eta\sqrt{\beta_1} \tilde{c}_{k,t}/2}+\prob{\bar{\mathcal{E}}_{3k}}\right].
\end{align}
where the last inequality is from $w_{k,t}<1$ and $ c_{k,n}\geq \sqrt{\beta_1}\tilde{c}_{k,n}$ as shown in \cite{RB:21}.  

We shall investigate $\prob{\sigma^2/(T_{k,t}\sigma_0^2+\sigma^2)|\bar{r}_{0,t}-\mu_k|>\eta\sqrt{\beta_1}\tilde{c}_{k,t}/2}$ and $\prob{\bar{\mathcal{E}}_{3k}}$ respectively. 
From the properties of sub-Gaussian, 
$\bar{r}_{0,t}-\mu_0$ is sub-Gaussian with the parameters 
$$\nu_t=:\left[\sum\limits_{k=1}^K(1-w_{k,t})T_{k,t}\right]^{-2} \sum\limits_{k=1}^K(1-w_{k,t})^2T_{k,t}^2(\sigma_0^2+\nu^2/T_{k,t})\geq K^{-1}(\sigma_0^2+\nu^2/2),$$
where the inequality is from the Cauchy–Schwarz inequality and $T_{k,t}\geq 2$. 
It follows that 
\begin{align}\label{e31}
\prob{\sigma^2/(T_{k,t}\sigma_0^2+\sigma^2)|\bar{r}_{0,t}-\mu_k|>\eta\sqrt{\beta_1}\tilde{c}_{k,t}/2}
\leq & 2\exp\left(-\frac{2\beta_1\sigma_0^2(T_{k,n}\sigma_0^2+\sigma^2)K\log n}{4(1+4\rho)^2\sigma^2(\sigma_0^2+\nu^2/2)}\right)\notag\\
\leq & 2\exp\left(-\frac{\beta_1\sigma_0^2K\log n}{(1+4\rho)^2\sigma^2}\right),
\end{align}
where the first step is from the sub-Gaussian tail inequality and the second step is from $T_{k,t}\geq 2$ and $\sigma^2\geq \nu^2$

On the other hand, we have 
\begin{align*}%\label{decomk-v3}
\prob{\bar{\mathcal{E}}_{3k}}
&\leq 2\exp\left(-\frac{2T_{k,t}\eta^2\beta_1\sigma^2\sigma_0^2\log n}{4\nu^2(T_{k,t}\sigma_0^2+\sigma^2)}\right)\leq 2\exp\left(-\frac{\beta_1\sigma_0^2\log n}{\delta(1+4\rho)^2(2\sigma_0^2+\sigma^2)}\right),
\end{align*}
where the first step is from the sub-Gaussian tail inequality and noticing $\sigma^2=\nu^2/\delta$, and 
the last step is from $T_{k,t}\geq 2$. 
Denote $m=\beta_1(1+4\rho)^{-2}\sigma^{-2}\sigma_0^2$, we have 
\begin{align}\label{decomk-v3}
\prob{\bar{\mathcal{E}}}
&\leq 2Kn^{-mK+1}+2Kn^{-\frac{\sigma^2m}{\delta(2\sigma_0^2+\sigma^2)}+1}, 
\end{align}
Therefore, combing \eqref{prob-dk-v1}, \eqref{prob-dk-v2}, and \eqref{decomk-v2}, the theorem is proved. 

%\end{proof}

%%%%%%
\subsection{Proof of Theorem \ref{them-R-Bayes}}
\label{sec:proofthem2}
Let $\mu_{(1)}=\max\{\mu_k, k\in[K]\}$ and $\mu_{(K)}=\min\{\mu_k, k\in[K]\}$. 
We have that
\begin{align*}%\label{Rn-en}
\mathrm{sr}_n
= & \sum\limits_{k\neq i^*}\text{Pr}(J_n= k)\condE{\mu_* - \mu_{J_n}}{J_n= k}\notag\\
\leq & e_n\E{\mu_{(1)} - \mu_{(K)}}\leq 2\sigma_0\sqrt{2\log K}e_n,
\end{align*}
where the first inequality is from $\mu_{J_n}\geq \mu_{(K)}$, and the last step is due to the fact of $\E{\mu_{(1)}- \mu_{(K)}}\leq 2\sigma_0\sqrt{2\log K}$. 
Combing Theorem \ref{them-R}, the proof is concluded.

%%%%%%%%%%%%%%
\subsection{Lemmas}
\label{sec:lemma}

\begin{lemma}\label{TwoBoundOnC}
On the event $\mathcal{E}$, 
We have the following two results: 
for $k\neq 1$, 
\begin{align}
T_{k,t}&\leq 2+2\Delta_k^{-2}(1+\eta)^2\beta\sigma^2\log n\label{Tk_up1}\\
(1-\eta)\sqrt{\beta_1}\tilde{c}_{k,t} &\leq (1+\eta)\sqrt{\beta}\tilde{c}_{1,t}+\Delta_k\label{Tk_lowv2}.
\end{align}
\end{lemma}

\begin{proof}

\eqref{Tk_up1} is obviously true at time $t=2K+1$. 
We assume that it holds at time $t\geq 2K+1$. If $I_{t+1}\neq k$, then $T_{k,t+1}=T_{k,t}$, thus it still holds. 
If $I_{t+1}=k$, it means that $\hat{\mu}_{k,t}+c_{k,t}\geq \hat{\mu}_{1,t}+c_{1,t}$. 
Note that on $\mathcal{E}$, we have that 
$$\hat{\mu}_{1,t}+c_{1,t}\geq \mu_1 \text{, and } \hat{\mu}_{k,t}+c_{k,t}\leq \mu_k+(1+\eta)c_{k,t}.$$
They follows $$(1+\eta)c_{k,t}\geq \Delta_k.$$
Thus, $T_{k,t}\leq 2\Delta_k^{-2}(1+\eta)^2\beta\sigma^2\log n$ holds due to $\tau_{k,t}^2\leq \beta\sigma_0^2\sigma^2/(T_{k,t}\sigma_0^2+\sigma^2)$ shown in Lemma \ref{sec:tau}. 
By using $T_{k,t+1}=T_{k,t}+1$, we prove \eqref{Tk_up1}.

\eqref{Tk_lowv2} is obviously true at the initial time $t=2K+1$. 
We assume that it holds at time $t\geq 2K+1$. If $I_{t+1}\neq 1$, then $T_{1,t+1}=T_{1,t}$, thus it still holds.  
If $I_{t+1}= 1$, it means that 
$$\hat{\mu}_{k,t}+c_{k,t}\leq \hat{\mu}_{1,t}+c_{1,t}.$$ 
Note that on $\mathcal{E}$, we have that
\begin{equation*}
\hat{\mu}_{1,t}+c_{1,t}\leq \mu_1+(1+\eta)c_{1,t} \text{, and }
\hat{\mu}_{k,t}+c_{k,t}\geq \mu_k+(1-\eta)c_{k,t}. 
\end{equation*}
They follow
$$(1-\eta)c_{k,t} \leq (1+\eta)c_{1,t}+\Delta_k.$$ 
Since $c_{k,t}\geq \sqrt{\beta_1}\tilde{c}_{k,t}$ and $c_{1,t}\leq \sqrt{\beta}\tilde{c}_{1,t}$ shown in Lemma \ref{sec:tau}, we have that 
$$(1-\eta)\sqrt{\beta_1}\tilde{c}_{k,t} \leq (1+\eta)\sqrt{\beta}\tilde{c}_{1,t}+\Delta_k.$$ 
By using $T_{k,t+1}=T_{k,t}+1$, we prove \eqref{Tk_lowv2}.
\end{proof}

%%%%%%%%%%%%
%%%%%%%%%%%%
\begin{lemma}\label{sec:tau}
(Lemmas 1 \& 5 in \cite{RB:21})\begin{align*}
\frac{\sigma_0^2\sigma^2}{T_{k,t}\sigma_0^2+\sigma^2}(1+K^{-1}\sigma_0^2/(\sigma_0^{2}+\sigma^2))\leq \tau_{k,t}^2 \leq \frac{\sigma_0^2\sigma^2}{T_{k,t}\sigma_0^2+\sigma^2}(1+K^{-1}\sigma^2\sigma_0^{-2}).
\end{align*}
\end{lemma}

%%%%%%%%%%%%%%%%%%%%%%
%\newpage
%\section*{Bound}
%%%%%%%%%%%%%%%%%%%%%%
%%%%%

\subsection{Some Properties}
\label{sec:property}

{\bf Fact 1} (Gaussian tail bound)
Let $X$ be a Gaussian random variable, i.e., $X\sim \mathcal{N}(0,\sigma^2)$, 
then for all $\alpha>0$, 
\begin{equation*}%\label{MGF-inequality}
\text{Pr}(X\geq \alpha)\leq \exp\left(-\frac{\alpha^2}{2\sigma^2}\right).
\end{equation*}

{\bf Fact 2} (Joint distribution of ordered statistics)
Denote  $F(x)=\text{P}(X\leq x)$ and $f(x)$ as its density. 
Let $X_{(1)}\geq X_{(2)}\geq \dots \geq X_{(K)}$. 
For $x_1\geq x_2$, the density of $(X_{(1)},X_{(2)})$ is 
$$p(x_2,x_1)=n(n-1)F(x_2)^{n-2}f(x_2)f(x_1).$$

{\bf Fact 3} (Some results on Gaussian)
$$\frac{1}{\sqrt{2\pi}}\frac{t}{t^2+1}\exp(-t^2/2)\leq 1-\Phi(t)=\text{Pr}(X>t)\leq \frac{1}{t\sqrt{2\pi}}\exp(-t^2/2).$$
$$\text{Pr}(X_{(2)}<t)=[\Phi(t)]^n+n[1-\Phi(t)][\Phi(t)]^{n-1}.$$

%%%%%%%%%%%%%%%%%%%%%%%%%%
\subsection{More Results of Experiments}
\label{sec:app-simulation}

\textbf{Performance of the two-stage algorithm}
We show the empirical studies of the two-stage algorithm. 
For overall checking the performance, we check it under various $q=0.1,0.2,\dots,0.9$, we report their performance under Setup F4 in Figure \ref{fig:twostage}. 
The figure shows the Two-Stage algorithm is bad under various $q$.

\begin{figure*}[!ht]
\centering
\begin{subfigure}[b]{0.495\columnwidth}
\includegraphics[keepaspectratio,width=1\linewidth]{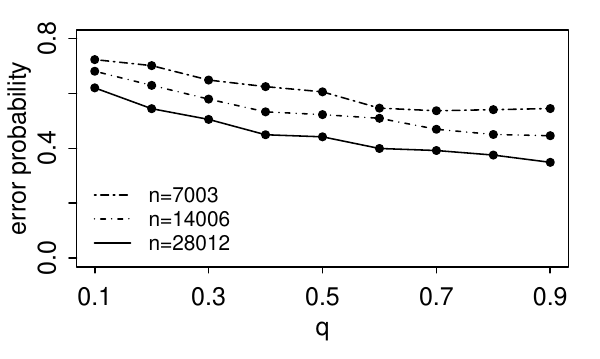}
\subcaption{Fixed Setup F4} 
\end{subfigure}
\begin{subfigure}[b]{0.495\columnwidth}
\includegraphics[keepaspectratio,width=1\linewidth]{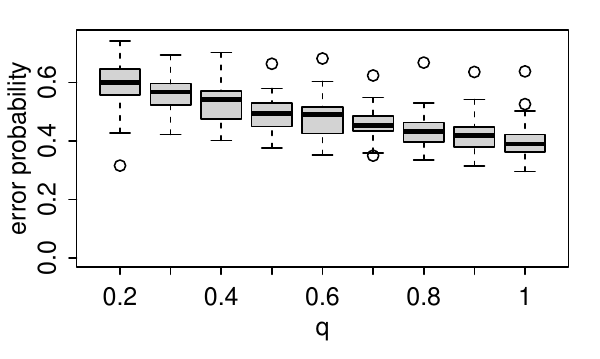}
\subcaption{Random Setup R2} 
\end{subfigure}
\caption{The error probability of the Two-Stage algorithm with various $q$ values, averaged over 1000 independent executions (results in standard deviations of less than 0.016).} 
\label{fig:twostage}
\end{figure*}

%%%%%%%%%%%%%%%%%%%%%%%%%%
\textbf{Impact of arm number $K$.}
We ran the experiments with
$n = 20, 40, 80$ arms in order to examine how the
performance of each algorithm scales as the number
of arms grow. We report the result on the arithmetic setting (Setup F4) in Figure \ref{fig:MoreArms}, where $K=40$ and 80 are shown. 
Comparing various $K$, the benefit of using \rue increases as $K$ increases. 

\begin{figure*}[!ht]
\centering
\includegraphics[keepaspectratio,width=1\linewidth]{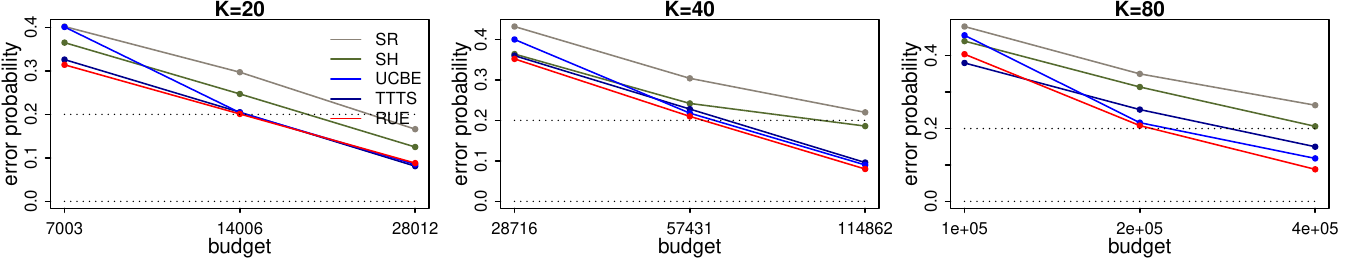}
\caption{The error probability of the different algorithms in Setup F4 with more arms, 40 and 80 arms (left and right subfigures respectively). The results are averaged over 1000 independent executions (all standard errors are less than 0.016 and not reported). For $K=80$, we set $N=400000$ as the maximal budget for the limit of resources when $2H$ is too big.}
\label{fig:MoreArms}
\end{figure*}

%%%%%

\end{document}